\documentclass[runningheads]{llncs}
\usepackage{graphicx, caption, subcaption}
\usepackage[pdfborder={0 0 0}]{hyperref}
\usepackage{url}

\usepackage[disable]{todonotes}
\usepackage{cite}
\usepackage{ntheorem}
\usepackage{amsmath,amssymb,mathtools}
\usepackage[algo2e,noend,linesnumbered,tworuled]{algorithm2e}\DontPrintSemicolon
\usepackage{tabto}
\usepackage{pgfplots}
\usepackage{filecontents}
\theoremstyle{definition}
\newtheorem{axiom}{Axiom}
\newcommand{\nearest}{{n_1}}
\newcommand{\second}{{n_2}}

\newcommand{\MS}{\ensuremath{\mathrm{\tilde{S}}}}
\newcommand{\ms}{\ensuremath{\mathrm{\tilde{s}}}}

\newcommand{\tsum}{\textstyle\sum\nolimits}
\newcommand{\mean}{\operatorname{mean}}
\newcommand{\argmax}{\operatorname{arg\,max}}
\newcommand{\argmin}{\operatorname{arg\,min}}
\xdefinecolor{Pa1}{RGB}{132, 184, 24} %
\xdefinecolor{Pa2}{RGB}{241, 135, 0} %
\xdefinecolor{Pa3}{RGB}{0, 155, 170} %
\xdefinecolor{Pa4}{RGB}{191, 2, 127} %
\xdefinecolor{Pa5}{RGB}{250, 220, 0} %
\xdefinecolor{Pa6}{RGB}{228, 103, 8} %
\pgfplotsset{compat=newest}

\begin{document}
\title{Clustering by Direct Optimization of the Medoid Silhouette\thanks{Part of the work on this paper has been supported by Deutsche Forschungsgemeinschaft (DFG) -- project number 124020371 -- within the Collaborative Research Center SFB 876 ``Providing Information by Resource-Constrained Analysis'' project A2.}}
\author{Lars Lenssen\orcidID{0000-0003-0037-0418} \and
Erich Schubert\orcidID{0000-0001-9143-4880}}
\authorrunning{L. Lenssen and E. Schubert}
\institute{TU Dortmund University, Informatik VIII, 44221 Dortmund, Germany 
\email{\{lars.lenssen,erich.schubert\}@tu-dortmund.de}}
\maketitle              %
\begin{abstract}
The evaluation of clustering results is difficult, highly dependent on the evaluated data set and the perspective of the beholder. There are many different clustering quality measures, which try to provide a general measure to validate clustering results. A very popular measure is the Silhouette. We discuss the efficient medoid-based variant of the Silhouette, perform a theoretical analysis of its properties, and provide two fast versions for the direct optimization. We combine ideas from the original Silhouette with the well-known PAM algorithm and its latest improvements FasterPAM. One of the versions guarantees equal results to the original variant and provides a run speedup of $O(k^2)$. In experiments on real data with 30000 samples and $k$=100, we observed a 10464$\times$ speedup compared to the original PAMMEDSIL algorithm.

\end{abstract}
\section{Introduction}
\begin{tikzpicture}[overlay, remember picture]
\node[red, yshift=-20mm, anchor=north, text width=115mm] at (current page.north) {
\textbf{Preprint} version. Please consult the final version of record instead:
\\
Lars Lenssen, Erich Schubert:
Clustering by Direct Optimization of the Medoid Silhouette.
Similarity Search and Applications (SISAP 2022)\\
\url{https://doi.org/10.1007/978-3-031-17849-8_15}
};
\end{tikzpicture}%
In cluster analysis, the user is interested in discovering previously
unknown structure in the data, as opposed to classification
where one predicts the known structure (labels) for new data points.
Sometimes, clustering can also be interpreted as data quantization and approximation,
for example $k$-means which aims at minimizing the sum of squared errors when approximating
the data with $k$ average vectors, spherical $k$-means which aims to maximize the cosine similarities
to the $k$ centers, and $k$-medoids which minimizes the sum of distances
when approximating the data by $k$ data points.
Other clustering approaches such as DBSCAN \cite{DBLP:conf/kdd/EsterKSX96,DBLP:journals/tods/SchubertSEKX17}
cannot easily be interpreted this way, but discover structure related
to connected components and density-based minimal spanning trees \cite{DBLP:conf/lwa/SchubertHM18}.

The evaluation of clusterings is a challenge, as there are no labels available.
While many internal (``unsupervised'', not relying on external labels)
evaluation measures were proposed such as the Silhouette~\cite{Rousseeuw/87a},
Davies-Bouldin index, the Variance-Ratio criterion, the Dunn index, and many more,
using these indexes for evaluation suffers from inherent problems.
Bonner \cite{DBLP:journals/ibmrd/Bonner64} noted that
``none of the many specific definitions [...] seems best in any general sense'',
and results are subjective ``in the eye of the beholder'' as noted by
Estivill-Castro~\cite{DBLP:journals/sigkdd/Estivill-Castro02}.
While these claims refer to clustering methods, not evaluation methods,
we argue that these do not differ substantially:
each internal cluster evaluation method implies a clustering algorithm obtained %
by enumeration of all candidate clusterings, keeping the best.
The main difference between clustering algorithms and internal evaluation then
is whether or not we know an efficient optimization strategy. $K$-means 
is an optimization strategy for the sum of squares evaluation measure,
while the $k$-medoids algorithms PAM, and Alternating are different strategies for
optimizing the sum of deviations from a set of $k$ representatives.

In this article, we focus on the evaluation measure known as
Silhouette~\cite{Rousseeuw/87a}, %
and discuss an efficient algorithm to optimize a variant of this measure,
inspired by the well-known PAM algorithm \cite{Kaufman/Rousseeuw/87a, Kaufman/Rousseeuw/90c}
and FasterPAM~\cite{DBLP:journals/is/SchubertR21,DBLP:conf/sisap/SchubertR19}.

\section{Silhouette and Medoid Silhouette}

The Silhouette~\cite{Rousseeuw/87a} is a popular measure to evaluate clustering validity, and
performs very well in empirical studies \cite{DBLP:journals/pr/ArbelaitzGMPP13,DBLP:journals/pr/BrunSHLCSD07}.
For the given samples $X {=} \{x_1,\ldots,x_n\}$, a dissimilarity measure $d:X{\times} X\mapsto \mathbb{R}$,
and the cluster labels $L{=}\{l_1,\ldots,l_n\}$,
the Silhouette for a single element $i$ is calculated based on the average distance to its own cluster
$a_i$ and the smallest average distance to another cluster $b_i$ as:
\begin{align*}
s_i(X, d, L) &= \tfrac{b_i-a_i}{\max(a_i,b_i)}
\;\text{, where}\\
a_i &= \phantom{\min\nolimits_{k\neq l_i}\;} \mean \left\{d(x_i, x_j) \mid l_j = l_i\right\}
\\
b_i &= \min\nolimits_{k\neq l_i}\;\mean\left\{d(x_i, x_j) \mid l_j = k\right\}
\;.
\end{align*}
The motivation is that ideally, each point is much closer to the cluster it is assigned to,
than to another ``second closest'' cluster. For $b_i{\gg} a_i$, the Silhouette approaches 1,
while for points with $a_i{=}b_i$ we obtain a Silhouette of 0, and negative values can arise
if there is another closer cluster and hence $b_i{<}a_i$.
The Silhouette values $s_i$ can then be used to visualize the cluster quality by sorting
objects by label $l_i$ first, and then by descending $s_i$, to obtain the Silhouette plot.
However, visually inspecting the Silhouette plot is only feasible for small data sets,
and hence it is also common to aggregate the values into a single statistic, often referred to
as the Average Silhouette Width (ASW):
\begin{align*}
S(X, d, L) = \tfrac{1}{n} \tsum_{i=1}^n s_i(X, d, L)
\;.
\end{align*}
Hence, this is a function that maps a data set, dissimilarity, and cluster labeling to a real number,
and this measure has been shown to satisfy
desirable properties for clustering quality measures (CQM) by Ackerman and Ben-David \cite{DBLP:conf/nips/Ben-DavidA08}.

A key limitation of the Silhouette is its computational cost. It is easy to see that it requires
all pairwise dissimilarities, and hence takes $O(N^2)$ time to compute -- much more than popular
clustering algorithms such as $k$-means.

For algorithms such as $k$-means and $k$-medoids, a simple approximation to the Silhouette is
possible by using the distance to the cluster center respectively medoids
$M=\{M_1,\ldots,M_k\}$ instead of the average distance.
For this ``simplified Silhouette'' (which can be computed in $O(N k)$ time,
and which Van der Laan et al.~\cite{VanderLaan/03a} called medoid-based Silhouette)
we use
\begin{align*}
s_i'(X, d, M) &= \tfrac{b_i'-a_i'}{\max(a_i',b_i')}
\;\text{, where}\\
a_i' &= \phantom{\min\nolimits_{k\neq l_i}\;{}} d(x_i, M_{l_i})
\\
b_i' &= \min\nolimits_{k\neq l_i}\; d(x_i, M_k)
\;.
\end{align*}

If each point is assigned to the closest cluster center (optimal for $k$-medoids and the Silhouette),
we further know that $a_i'\leq b_i'$ and $s_i\geq 0$,
and hence this can further be simplified to
the \emph{Medoid Silhouette}
\begin{align*}
\tilde{s}_i(X, d, M) &= \tfrac{d_2(i)-d_1(i)}{d_2(i)}
= 1 - \tfrac{d_1(i)}{d_2(i)}
\;.
\end{align*}
where $d_1$ is the distance to the closest and $d_2$ to the second closest center in~$M$.
For $d_1(i)=d_2(i)=0$, we add a small $\varepsilon$ to $d_2(i)$ to get $\tilde{s}=1$.
The Average Medoid Silhouette (AMS) then is defined as the average
\begin{align*}
\tilde{S}(X, d, M) = \tfrac{1}{n} \tsum_{i=1}^n \tilde{s}_i(X, d, M)
\;.
\end{align*}
It can easily be seen that the optimum clustering is the (assignment of points to the) set of medoids
such that we minimize an ``average relative loss``:
\begin{align*}
\argmax_M \tilde{S}(X, d, M) = \argmin_M \mean_i \tfrac{d_1(i)}{d_2(i)}
\;.
\end{align*}
For clustering around medoids, we impose the restriction $M\subseteq X$;
which has the benefit of not restricting the input data to be numerical
(e.g., $X\subset\mathbb{R}^d$, as in $k$-means), and allowing non-metric
dissimilarity functions~$d$.

\section{Related Work}

The Silhouette~\cite{Rousseeuw/87a} was originally proposed along with
Partitioning Around Medoids (PAM,~\cite{Kaufman/Rousseeuw/87a, Kaufman/Rousseeuw/90c}),
and indeed $k$-medoids already does a decent job at finding a good solution,
although it does optimize a different criterion (the sum of total deviations).
Van der Laan et al.~\cite{VanderLaan/03a} proposed to optimize the Silhouette by substituting
the Silhouette evaluation measure into the PAM SWAP procedure (calling this PAMSIL).
Because they recompute the loss function each time (as opposed to PAM, which computes the change),
the complexity of PAMSIL is $O(k(N-k) N^2)$,
since for each of $k\cdot (N-k)$ possible swaps, the Silhouette is computed in $O(N^2)$.
Because this yields a very slow clustering method, they also considered the Medoid Silhouette
instead (PAMMEDSIL), which only needs $O(k^2(N-k)N)$ time (but still considerably more than PAM).

Schubert and Rousseeuw \cite{DBLP:journals/is/SchubertR21,DBLP:conf/sisap/SchubertR19} recently improved the PAM method,
and their FastPAM approach reduces the cost of PAM by a factor of $O(k)$,
making the method $O(N^2)$ by the use of a shared accumulator to avoid the innermost loop.
In this work, we will combine ideas from this algorithm with the PAMMEDSIL approach above,
to optimize the Medoid Silhouette with a swap-based local search, but a run time comparable
to FastPAM. But we will first perform a theoretical analysis of the properties of the Medoid
Silhouette, to show that it is worth exploring as an alternative to the original Silhouette.

\section{Axiomatic Characterization of Medoid Clustering}
We follow the axiomatic approach of Ackerman and Ben-David~\cite{DBLP:conf/nips/Ben-DavidA08},
to prove the value of using the Average Medoid Silhouette (AMS) as a clustering quality measure (CQM).
Kleinberg~\cite{Kleinberg/Jon/02a} defined three axioms for clustering functions
and argued that no clustering algorithm can satisfy these desirable properties at the same time,
as they contradict.
Because of this, Ackermann and Ben-David~\cite{DBLP:conf/nips/Ben-DavidA08}
weaken the original Consistency Axiom and extract four axioms for clustering quality measures:
\emph{Scale Invariance} and \emph{Richness} are defined analogously to the Kleinberg Axioms. We redefine the CQM axioms~\cite{DBLP:conf/nips/Ben-DavidA08} for medoid-based clustering.
\begin{definition}
\label{d1}
For given data points $X = \{x_1,\ldots,x_n\}$ with a set of $k$ medoids $M = \{m_1,\ldots,m_k\}$ and a dissimilarity $d$, we write $x_i \sim_M x_{i'}$ whenever $x_i$ and~$x_{i'}$ have the same nearest medoid $\nearest(i) \subseteq M$, otherwise $x_i \not\sim_M x_{i'}$.
\end{definition}
\begin{definition}
\label{d2}
Dissimilarity $d'$ is an M-consistent variant of $d$, if $d'(x_i, x_{i'}) \leq d(x_i, x_{i'})$ for $x_i \sim_M x_{i'}$, and $d'(x_i, x_{i'}) \geq d(x_i, x_{i'})$ for $x_i \not\sim_M x_{i'}$.
\end{definition}
\begin{definition}
\label{d3}
Two sets of medoids $M, M' \subseteq X$ with a distance function $d$ over~$X$,
are isomorphic, if there exists a distance-preserving isomorphism $\phi : X \to X$,
such that for all $x_i, x_{i'} \in X$, $x_i \sim_M x_{i'}$ if and only if $\phi(x_i) \sim_{M'} \phi(x_{i'})$.
\end{definition}
\begin{axiom}[Scale Invariance]
\label{a1}
A medoid-based clustering quality measure~$f$ satisfies scale invariance if for every set of medoids $M \subseteq X$ for $d$, and every positive~$\lambda$, $f(X, d, M) = f(X, \lambda d, M)$.
\end{axiom}
\begin{axiom}[Consistency]
\label{a2}
A medoid-based clustering quality measure~$f$ satisfies consistency if for a set of medoids $M \subseteq X$ for $d$, whenever $d'$ is an \mbox{M-consistent} variant of~$d$, then $f(X, d', M) \geq f(X, d, M)$. 
\end{axiom}
\begin{axiom}[Richness]
\label{a3}
A medoid-based clustering quality measure~$f$ satisfies richness if for each set of medoids $M \subseteq X$, there exists a distance function $d$ over~$X$ such that $M = \argmax_{M}f(X, d, M)$.
\end{axiom}
\begin{axiom}[Isomorphism Invariance]
\label{a4}
A medoid-based clustering quality measure~$f$ is isomorphism-invariant if for all
sets of medoids $M, M' \subseteq X$ with distance $d$ over~$X$ where $M$ and $M'$ are isomorphic, $f(X, d, M) = f(X, d, M')$. 
\end{axiom}

Batool and Hennig~\cite{Batool/Hennig/21a} prove that the ASW
satisfies the original CQM axioms.
We prove the first three adapted axioms for the Average Medoid Silhouette.
The fourth, Isomorphism Invariance, is obviously fulfilled, since AMS is based only on dissimilarites, just as the ASW~\cite{Batool/Hennig/21a}.
\begin{theorem}
The AMS is a \emph{scale invariant} clustering quality measure.
\end{theorem}
\begin{proof}
If we replace $d$ with $\lambda d$, both $d_1(i)$ and $d_2(i)$ are multiplied by $\lambda$, and the term will cancel out. Hence, $\tilde{s}_i$ does not change for any $i$:
\begin{align*}
\tilde{S}(X, \lambda d, M) &= \tfrac{1}{n} \tsum_{i=1}^n \tilde{s}_i(X, \lambda d, M) 
= \tfrac{1}{n} \tsum_{i=1}^n \tfrac{\lambda d_2(i)-\lambda d_1(i)}{\lambda d_2(i)} \\
&= \tfrac{1}{n} \tsum_{i=1}^n \tfrac{ d_2(i)- d_1(i)}{ d_2(i)} 
= \tfrac{1}{n} \tsum_{i=1}^n \tilde{s}_i(X, d, M) 
= \tilde{S}(X, d, M)
\;.
\end{align*}
\end{proof}
\begin{theorem}
The AMS is a \emph{consistent} clustering quality measure.
\end{theorem}
\begin{proof}
Let dissimilarity $d'$ be a M-consistent variant of $d$.
By Definition~\ref{d2}: $d'(x_i, x_{i'} ) \leq d(x_i, x_{i'} )$ for all $x_i \sim_{M} x_{i'}$, and
$\min_{x_i \not\sim_{M} x_{i'}}d'(x_i, x_{i'} ) \geq \min_{x_i \not\sim_{M} x_{i'}}d(x_i, x_{i'} )$.
This implies for all $i \in \mathbb{N}$:
$d'_1(i) \leq d_1(i), d'_2(i) \geq d_2(i)$
and it follows:
\begin{align*}
\tfrac{d_1(i)}{ d_2(i)} - \tfrac{ d'_1(i)}{ d'_2(i)} &\geq 0
\quad\Leftrightarrow\quad  %
\tfrac{ d'_2(i)- d'_1(i)}{ d'_2(i)} - \tfrac{ d_2(i)- d_1(i)}{ d_2(i)} \geq 0
\end{align*}
which is equivalent to $\forall_i \; \tilde{s}_i(X, d', M) \geq \tilde{s}_i(X, d, M)$,
hence $\tilde{S}(X, d', M) \geq \tilde{S}(X, d, M)$, i.e., AMS is a consistent clustering quality measure.
\end{proof}
\begin{theorem}
The AMS is a \emph{rich} clustering quality measure.
\end{theorem}
\begin{proof}
We can simply encode the desired set of medoids $M$ in our dissimilarity~$d$.
We define $d(x_i,x_j)$ such that it is~0 if trivially $i=j$,
or if $x_i$ or $x_j$ is the first medoid $m_1$ and the other is not a medoid itself.
Otherwise, let the distance be~1.

For $M$ we then obtain $\tilde{S}(X, d, M)=1$,
because $d_1(i)=0$ for all objects, as either $x_i$ is a medoid itself, or can
be assigned to the first medoid~$m_1$. This is the maximum possible Average Medoid Silhouette.
Let $M'\neq M$ be any other set of medoids. Then there exists at least one missing
$x_i\in M\setminus M'$. For this object $\tilde{s}_i(X, d, M)=0$ (as its distance to all other objects is 1, and it is not in~$M'$), and hence $\tilde{S}(X, d, M')<1=\tilde{S}(X, d, M)$.
\end{proof}

\section{Direct Optimization of Medoid Silhouette}
PAMSIL~\cite{VanderLaan/03a} is a modification of PAM~\cite{Kaufman/Rousseeuw/87a, Kaufman/Rousseeuw/90c}
to optimize the ASW.
For PAMSIL, Van der Laan~\cite{VanderLaan/03a} adjusts the SWAP phase of PAM by always performing the SWAP
that provides the best increase in the ASW.
When no further improvement is found, a (local) maximum of the ASW has been achieved.
However, where the original PAM efficiently computes only the change in its loss
(in $O(N-k)$ time for each of $(N-k)k$ swap candidates),
PAMSIL computes the entire ASW in $O(N^2)$ for every candidate,
and hence the run time per iteration increases to $O(k(N-k)N^2)$.
For a small $k$, this yields a run time that is cubic in the number of objects $N$,
and the algorithm may need several iterations to converge.
\subsection{Naive Medoid Silhouette Clustering}
PAMMEDSIL~\cite{VanderLaan/03a} uses the Average Medoid Silhouette (AMS) instead,
which can be evaluated in only $O(Nk)$ time. This yields a SWAP run time
of $O(k^2(N-k)N)$ (for small $k\ll N$ only quadratic in~$N$).
\begin{algorithm2e}[tb!]
\caption{PAMMEDSIL SWAP: Iterative improvement}
\label{alg1}
\SetKwBlock{Repeat}{repeat}{}
\SetKw{Break}{break}
$S' \gets $ Simplified Silhouette sum of the initial solution $M$\;
\Repeat{
    $(S'_*, M_*)\gets(0,$null$)$\;
    \ForEach(\tcp*[f]{each medoid}\label{alg1-loop1}){$m_i\in M=\{m_1,\ldots,m_k\}$} {
        \ForEach(\tcp*[f]{each non-medoid}\label{alg1-loop2}){$x_j\notin\{m_1,\ldots,m_k\}$} {
            $(S',M') \gets (0, M \setminus \{m_i\} \cup \{x_j\})$\;
            \ForEach(\label{alg1-loop3}){$x_o\in X=\{x_1,\ldots,x_n\}$} {
                $S' \gets S' + s_o'(X,d,M')$\tcp*{ Simplified Silhouette}
            }
            \lIf(\tcp*[f]{keep best swap found}\label{alg1-if1}) {$S' > S'_*$} {
        		$(S'_*, M_*)\gets( S',M')$%
        	}
        }
    }
    \lIf{$S'_*\geq S'$}{\Break}
    $(S',M) \gets (S'_*,M_*)$\tcp*[r]{perform swap}
}
\Return {$(S' / N,M)$}\;
\end{algorithm2e}
As Schubert and Rousseeuw \cite{DBLP:journals/is/SchubertR21,DBLP:conf/sisap/SchubertR19} were able to reduce the run time of PAM
to $O(N^2)$ per iteration, %
we modify the PAMMEDSIL approach accordingly to obtain a similar speedup.
The SWAP algorithm of PAMMEDSIL is shown in Algorithm~\ref{alg1}.

\subsection{Finding the Best Swap}
\label{sec52}
We first bring PAMMEDSIL up to par with regular PAM.
The trick introduced with PAM is to compute the change in loss instead of recomputing the loss,
which can be done in $O(N-k)$ instead of $O(k(N-k))$ time if we store the distance to the
nearest and second centers, as the latter allows us to compute the change if the current nearest center is removed efficiently.
In the following, we omit the constant parameters $X$ and $d$ for brevity.
We denote the previously nearest medoid of $i$ as $\nearest(i)$,
and $d_1(i)$ is the (cached) distance to it. We similarly define $\second(i)$, $d_2(i)$, and $d_3(i)$
with respect to the second and third nearest medoid. We briefly use $d_1'$ and $d_2'$ to denote
the new distances for a candidate swap.
For the Medoid Silhouette, we can compute the change when swapping medoids
$m_i\in\{m_1,\ldots,m_k\}$ with non-medoids $x_j\notin\{m_1,\ldots,m_k\}$:
\begin{align*}
\Delta\MS &= \tfrac{1}{n} \tsum_{o=1}^n \Delta\ms_o(M,m_i,x_j)
\\
\Delta\ms_o(M,m_i,x_j) &=  \ms_o(M \setminus \{m_i\} \cup \{x_j\}) - \ms_o(M)
\\
&= \tfrac{d_2'(i) - d_1'(i)}{d_2'(i)} - \tfrac{d_2(i)-d_1(i)}{d_2(i)} = \tfrac{d_1(i)}{d_2(i)} - \tfrac{d_1'(i)}{d_2'(i)}
\;.
\end{align*}
Clearly, we only need the distances to the closest and second closest center,
before and after the swap. Instead of recomputing them, we exploit that only one
medoid can change in a swap. By determining the new values of $d_1'$ and $d_2'$
using cached values only, we can save a factor of $O(k)$ on the run time.

In the PAM algorithm (where the change would be simply $d_1'-d_1$),
the distance to the \emph{second} nearest is cached in order to compute the loss change
if the current medoid is removed, without having to consider all $k-1$ other medoids:
the point is then either assigned to the new medoid, or its former second closest.
To efficiently compute the change in Medoid Silhouette, we have to take this one step further,
and additionally need to cache the identity of the second closest center
and the distance to the \emph{third} closest center (denoted~$d_3$).
This is beneficial if, e.g., the nearest medoid is replaced.
Then we may have, e.g., $d_1'=d_2$ and $d_2'=d_3$, if we can distinguish these cases.

The change in Medoid Silhouette is then computed roughly as follows:
(1) If the new medoid is the new closest, the second closest is either the former nearest, or the second nearest (if the first was replaced).
(2) If the new medoid is the new second closest, the closest either remains the former nearest, or the second
nearest (if the first was replaced).
(3) If the new medoid is neither, we may still have replaced the closest or second closest;
in which case the distance to the third nearest is necessary to compute the new Silhouette.
Putting all the cases (and sub-cases) into one equation becomes a bit messy, and hence we
opt to use pseudocode in Algorithm~\ref{alg:change} instead of an equivalent mathematical notation.
\begin{algorithm2e}[tb]
\caption{Change in Medoid Silhouette, $\Delta\ms_o(M,m_i,x_j)$}
\label{alg:change}
\If(\tcp*[f]{nearest is replaced}){$m_i=\nearest(o)$}{
  \lIf(\tcp*[f]{xj is new nearest}){$d(o,j)< d_2(o)$}{
    \Return \tabto*{48mm} $\frac{d_1(o)}{d_2(o)}-\frac{d(o,j)}{d_2(o)}$
  }
  \lIf(\tcp*[f]{xj is new second}){$d(o,j)< d_2(o)$}{
    \Return \tabto*{48mm} $\frac{d_1(o)}{d_2(o)}-\frac{d_2(o)}{d(o,j)}$
  }
  \lElse {
    \Return \tabto*{48mm} $\frac{d_1(o)}{d_2(o)}-\frac{d_2(o)}{d_3(o)}$
  }
}
\ElseIf(\tcp*[f]{second nearest is replaced}){$m_i=\second(o)$}{
  \lIf(\tcp*[f]{xj is new nearest}){$d(o,j)< d_1(o)$}{
    \Return \tabto*{48mm} $\frac{d_1(o)}{d_2(o)}-\frac{d(o,j)}{d_1(o)}$
  }
  \lIf(\tcp*[f]{xj is new second}){$d(o,j)< d_3(o)$}{
    \Return \tabto*{48mm} $\frac{d_1(o)}{d_2(o)}-\frac{d_1(o)}{d(o,j)}$
  }
  \lElse {
    \Return \tabto*{48mm} $\frac{d_1(o)}{d_2(o)}-\frac{d_1(o)}{d_3(o)}$
  }
}
\Else{
  \lIf(\tcp*[f]{xj is new nearest}){$d(o,j)< d_1(o)$}{
    \Return \tabto*{48mm} $\frac{d_1(o)}{d_2(o)}-\frac{d(o,j)}{d_1(o)}$
  }
  \lIf(\tcp*[f]{xj is new second}){$d(o,j)< d_2(o)$}{
    \Return \tabto*{48mm} $\frac{d_1(o)}{d_2(o)}-\frac{d_1(o)}{d(o,j)}$
  }
  \lElse {
    \Return \tabto*{48mm} 0
  }
}
\end{algorithm2e}
Note that the first term is always the same (the previous loss), except for the last case,
where it canceled out via $0=\frac{d_1(o)}{d_2(o)}-\frac{d_1(o)}{d_2(o)}$.
As this is a frequent case, it is beneficial to not have further computations here
(and hence, to compute the change instead of computing the loss).
Clearly, this algorithm runs in $O(1)$ if $n_1(o)$, $n_2(o)$, $d_1(o)$, $d_2(o)$, and $d_3(o)$
are known. We also only compute $d(o,j)$ once.
Modifying PAMMEDSIL (Algorithm~\ref{alg1}) to use this computation yields a run time
of $O(k(N-k)N)$ to find the best swap, i.e., already $O(k)$ times faster.
But we can further improve this approach.\todo{Das ist momentan als ``PAMMEDSIL'' in den Experimenten, oder?
Dann benennen, und ``naive'' hinzufügen?}

\subsection{Fast Medoid Silhouette Clustering}

We now integrate an acceleration added to the PAM algorithm by
Schubert and Rousseeuw~\cite{DBLP:conf/sisap/SchubertR19,DBLP:journals/is/SchubertR21},
that exploits redundancy among the loop over the $k$ medoids to replace.
For this, the loss change $\Delta\MS(m_i,x_j)$ is split into multiple components:
(1)~the change by removing medoid $m_i$ (without choosing a replacement),
(2)~the change by adding $x_j$ as an additional medoid, and
(3)~a correction term if both operations occur at the same time.
The first factors can be computed in $O(kN)$, the second in $O(N(N-k))$, and the last
factor is~0 if the removed medoid is neither of the two closest, and hence is also in $O(N^2)$.
This then yields an algorithm that finds the best swap in $O(N^2)$, again $O(k)$ times faster.

The first terms (the removal of each medoid $m_i\in M$) are computed as:
\begin{align}
\Delta\MS^{-m_i} =& \sum\nolimits_{\nearest(o)=i} \tfrac{d_1(o)}{d_2(o)}-\tfrac{d_2(o)}{d_3(o)}
+ \sum\nolimits_{\second(o)=i} \tfrac{d_1(o)}{d_2(o)}-\tfrac{d_1(o)}{d_3(o)}
\;, \label{eq:removal-mi}
\shortintertext{
while for the second we compute the addition of a new medoid $x_j\not\in M$
}
\Delta\MS^{+x_j} =& \sum_{o=1}^n \begin{cases}
\frac{d_1(o)}{d_2(o)}-\frac{d(o,j)}{d_1(o)}      & \text{if }d(o,j)< d_1(o) \\
\frac{d_1(o)}{d_2(o)}-\frac{d_1(o)}{d(o,j)}     & \text{else if }d(o,j)< d_2(o) \\
0     & \text{otherwise}
\;.
\end{cases}
\notag
\shortintertext{
Combining these yields the change:
}
    \Delta\MS(m_i,x_j) =& \Delta\MS^{+x_j} + \Delta\MS^{-m_i} \notag\\ 
    &+ \sum_{\substack{o \text{ with}\\\nearest(o)=i}} \begin{cases}
    \frac{d(o,j)}{d_1(o)}+\frac{d_2(o)}{d_3(o)}-\frac{d_1(o)+d(o,j)}{d_2(o)} & \text{if }d(o,j)< d_1(o)\\ 
    \frac{d_1(o)}{d(o,j)}+\frac{d_2(o)}{d_3(o)}-\frac{d_1(o)+d(o,j)}{d_2(o)} & \text{else if }d(o,j)< d_2(o) \\
    \frac{d_2(o)}{d_3(o)}-\frac{d_2(o)}{d(o,j)}     & \text{else if }d(o,j)< d_3(o) \\
    0     & \text{otherwise} 
    \end{cases} \notag\\
    &+ \sum_{\substack{o \text{ with}\\\second(o)=i}} \begin{cases}
    \frac{d_1(o)}{d_3(o)}-\frac{d_1(o)}{d_2(o)}     & \text{if }d(o,j)< d_1(o) \\ 
    \frac{d_1(o)}{d_3(o)}-\frac{d_1(o)}{d_2(o)}     & \text{else if }d(o,j)< d_2(o) \\
    \frac{d_1(o)}{d_3(o)}-\frac{d_1(o)}{d(o,j)}     & \text{else if }d(o,j)< d_3(o) \\
    0     & \text{otherwise} 
\;.
    \end{cases}
\notag
\end{align}
It is easy to see that the additional summands can be computed by iterating over all objects $x_o$,
and adding their contributions to accumulators for $n_1(o)$ and $n_2(o)$.
As each object $o$ contributes to exactly two cases, the run time is $O(N)$.
\begin{algorithm2e}[tb!]
\caption{FastMSC: Improved SWAP algorithm}
\label{alg:fastpms}
\SetKwBlock{Repeat}{repeat}{}
\SetKw{BreakOuterLoopIf}{break outer loop if}
\Repeat{
\lForEach(\label{alg:fastpms-loop1}){$x_o$}{compute $\nearest(o), \second(o), d_1(o), d_2(o), d_3(o)$} 
    $\Delta\MS^{-m_1},\ldots,\Delta\MS^{-m_i} \gets$ compute loss change removing $m_i$ using \eqref{eq:removal-mi}\;
    $(\Delta\MS^*, m^*, x^*)\gets(0,$null$,$null$)$\;
    \ForEach(\tcp*[f]{each non-medoid}\label{alg1-loop2}){$x_j\notin\{m_1,\ldots,m_k\}$} {
        $\Delta\MS_i,\ldots,\Delta\MS_k\gets(\Delta\MS^{-m_1},\ldots,\Delta\MS^{-m_i})$\label{alg:fastpmsl6}\tcp*[r]{use removal loss}
        $\Delta\MS^{+x_j}\gets0$\label{alg:fastpmsl7}\tcp*[r]{initialize shared accumulator}
        \ForEach(\label{alg1-loop3}){$x_o\in\{x_1,\ldots,x_n\}$} {
            $d_{oj}\gets d(x_o,x_j)$\tcp*[r]{distance to new medoid}
            \If(\tcp*[f]{new closest}\label{alg:fastpmsl10}) {$d_{oj} < d_1(o)$} {
            $\Delta\MS^{+x_j}\gets\Delta\MS^{+x_j}+d_1(o)/d_2(o)-d_{oj}/d_1(o)$\label{alg:fastpmsl11}\;
        	$\Delta\MS_{\nearest(o)}\gets \Delta\MS_{\nearest(o)}+ d_{oj}/d_1(o) + d_2(o)/d_3(o) - \tfrac{d_1(o)+d_{oj}}{d_2(o)}$\label{alg:fastpmsl12}\;
        	$\Delta\MS_{\second(o)}\gets \Delta\MS_{\second(o)}+d_1(o)/d_3(o) - d_1(o)/d_2(o)$\;
            }
            \ElseIf(\tcp*[f]{new first/second closest}\label{alg:fastpms-if2}) {$d_{oj} < d_2(o)$} {
            $\Delta\MS^{+x_j}\gets\Delta\MS^{+x_j}+d_1(o)/d_2(o)-d_1(o)/d_{oj}$\label{alg:fastpmsl15}\;
        	$\Delta\MS_{\nearest(o)}\gets \Delta\MS_{\nearest(o)}+d_1(o)/d_{oj} + d_2(o)/d_3(o) - \tfrac{d_1(o)+d_{oj}}{d_2(o)}$\label{alg:fastpmsl16}\;
        	$\Delta\MS_{\second(o)}\gets \Delta\MS_{\second(o)}+d_1(o)/d_3(o) - d_1(o)/d_2(o)$\;
            }
            \ElseIf(\tcp*[f]{new second/third closest}\label{alg:fastpms-if2}) {$d_{oj} < d_3(o)$} {
        	$\Delta\MS_{\nearest(o)}\gets \Delta\MS_{\nearest(o)}+d_2(o)/d_3(o) - d_2(o)/d_{oj}$\;
        	$\Delta\MS_{\second(o)}\gets \Delta\MS_{\second(o)}+d_1(o)/d_3(o) - d_1(o)/d_{oj}$\label{alg:fastpmsl20}\;
            }
        }
        $i\gets \text{argmax}\Delta\MS_i$\;
        $\Delta\MS_i\gets\Delta\MS_i+\Delta\MS^{+x_j}$\;
        \lIf(\label{alg1-if1}) {$\Delta\MS_i > \Delta\MS^*$} {
        	$(\Delta\MS^*, m^*, x^*)\gets(\Delta\MS,m_i,x_j)$
        }
    }
    \BreakOuterLoopIf{$\Delta\MS^*\leq0$}\;
    swap roles of medoid $m^*$ and non-medoid $x^*$\tcp*[r]{perform swap}
    $\MS\gets\MS+\Delta\MS^*$\;
}
\Return{$\MS,M$}\;
\end{algorithm2e}
This then gives Algorithm~\ref{alg:fastpms}, which computes
$\Delta\MS^{+x_j}$ along with the sum of $\Delta\MS^{-m_i}$ and
these correction terms in an accumulator array.
The algorithm needs $O(k)$ memory for the accumulators in the loop,
and $O(N)$ additional memory to store the cached $n_1$, $n_2$, $d_1$, $d_2$, and $d_3$ for each object.

This algorithm gives the same result,\todo{ggf. camera-ready wieder numerical differences erwähnen}
but FastMSC (``Fast Medoid Silhouette Clustering'') is $O(k^2)$ faster than the naive PAMMEDSIL.

\subsection{Eager Swapping and Random Initialization}

We can now integrate further improvements by
Schubert and Rousseeuw~\cite{DBLP:journals/is/SchubertR21}.
Because doing the best swap (steepest descent) does not appear to guarantee finding better solutions,
but requires a pass over the entire data set for each step,
we can converge to local optima much faster if we perform every swap that yields an improvement,
even though this means we may repeatedly replacing the same medoid. For PAM they called this
eager swapping, and named the variant FasterPAM.
This does not improve theoretical run time (the last iteration will always require a pass over the
entire data set to detect convergence), but empirically reduces the number of iterations substantially.
It will no longer find the same results, but there is no evidence that a steepest descent is beneficial
over choosing the first descent found.
The main downside to this is, that it increases the dependency on the data ordering, and hence is best
used on shuffled data when run repeatedly.
Similarly, we will study a variant that eagerly performs the first swap that improves the AMS
as FasterMSC (``Fast and Eager Medoid Silhouette Clustering'').

Also, the classic initialization with PAM BUILD now becomes the performance bottleneck,
and Schubert and Rousseeuw~\cite{DBLP:journals/is/SchubertR21} showed that random initialization
in combination with eager swapping works very well.

\section{Experiments}
We next evaluate clustering quality, to show the benefits of optimizing AMS.
We report both AMS and ASW, as well as the supervised measures Adjusted Random Index (ARI) and
Normalized Mutual Information (NMI).
Afterward, we study the scalability,
to verify the expected speedup for our algorithm FastMSC.

\subsection{Data Sets}
Since it became possible to map gene expression at the single-cell level by RNA sequencing, clustering on these has become a popular task, and Silhouette is a popular evaluation measure there.
Single-cell RNA sequencing (scRNA-seq) provides high-dimensional data that requires appropriate preprocessing to extract information.
After extraction of significant genes, these marker genes are validated by clustering of proper cells.
\begin{figure}[tb!]
\begin{subfigure}{0.45\textwidth}\centering
\begin{tikzpicture}[font=\small]
\begin{axis}[unit vector ratio*=1 1 1, width=1.1\textwidth, xmin = -100, xmax = 150, ymin = -100, ymax = 150, xlabel={PC1}, ylabel={PC2}, xticklabels={,,}, yticklabels={,,}]
\addplot +[only marks, mark=text, text mark=1, mark options={scale=0.6}, Pa1] table [x=b, y=c, col sep=comma] {results/kolodziejczyk/counttable_al_2i.csv};\label{plot1_l_1}
\addplot +[only marks, mark=text, text mark=2, mark options={scale=0.6}, Pa2] table [x=b, y=c, col sep=comma] {results/kolodziejczyk/counttable_al_a2i.csv};\label{plot1_l_2}
\addplot +[only marks, mark=text, text mark=3, mark options={scale=0.6}, Pa3] table [x=b, y=c, col sep=comma] {results/kolodziejczyk/counttable_al_lif.csv};\label{plot1_l_3}
\addlegendimage{/pgfplots/refstyle=plot1_l_1}\addlegendentry{2i}
\addlegendimage{/pgfplots/refstyle=plot1_l_2}\addlegendentry{a2i}
\addlegendimage{/pgfplots/refstyle=plot1_l_3}\addlegendentry{lif}
\end{axis}
\end{tikzpicture}
\caption{Kolodziejczyk et al.~\cite{Kolodziejczyk/15a}}
\label{plot2_1}
\end{subfigure}
\begin{subfigure}{0.45\textwidth}\centering
\begin{tikzpicture}[font=\small]
\begin{axis}[unit vector ratio*=1 1 1, width=1.1\textwidth, xmin = -30, xmax = 70, ymin = -50, ymax = 50, xlabel={PC2}, ylabel={PC1}, xticklabels={,,}, yticklabels={,,}]
\addplot +[only marks, mark=text, text mark=1, mark options={scale=0.6}, Pa1] table [x=c, y=b, col sep=comma] {results/klein/counttable_klein_lif.csv};\label{plot2_l_1}
\addplot +[only marks, mark=text, text mark=2, mark options={scale=0.6}, Pa2] table [x=c, y=b, col sep=comma] {results/klein/counttable_klein_2d.csv};\label{plot2_l_2}
\addplot +[only marks, mark=text, text mark=4, mark options={scale=0.6}, Pa3] table [x=c, y=b, col sep=comma] {results/klein/counttable_klein_4d.csv};\label{plot2_l_3}
\addplot +[only marks, mark=text, text mark=7, mark options={scale=0.6}, Pa4] table [x=c, y=b, col sep=comma] {results/klein/counttable_klein_7d.csv};\label{plot2_l_4}
\addlegendimage{/pgfplots/refstyle=plot2_l_1}\addlegendentry{lif}
\addlegendimage{/pgfplots/refstyle=plot2_l_2}\addlegendentry{+2 days}
\addlegendimage{/pgfplots/refstyle=plot2_l_3}\addlegendentry{+4 days}
\addlegendimage{/pgfplots/refstyle=plot2_l_4}\addlegendentry{+7 days}
\end{axis}
\end{tikzpicture}
\caption{Klein et al.~\cite{Klein/15a}}
\label{plot2_2}
\end{subfigure}
\caption{Different kind of mouse embryonic stem cells (mESCs). For both data sets we have done PCA and plot the first two principal components. (a) shows 704 mESCs grown in three different conditions and (b) 2717 mESCs at the moment of LIF withdrawal, 2 days after, 4 days after, and 7 days after.}
\label{plot2}
\end{figure}
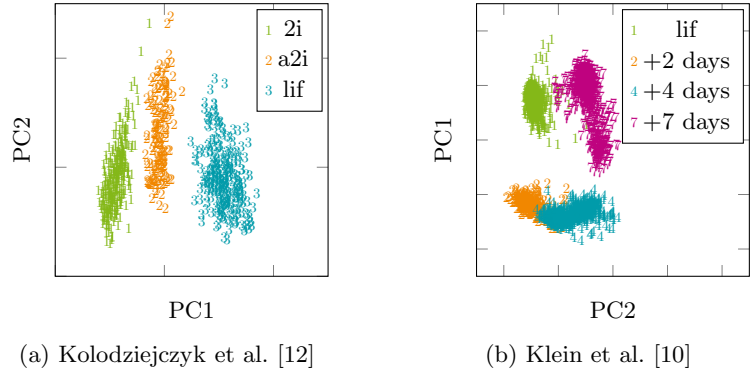

We explore two larger sample size (by scRNA standards) scRNA-sequencing data sets of mouse embryonic stem cells (mESCs) publicly available.
Kolodziejczyk et al.~\cite{Kolodziejczyk/15a} studied 704 mESCs with 38561 genes grown in three different conditions (2i, a2i and serum).
Klein et al.~\cite{Klein/15a} worked on the influence leukemia inhibitory factor (LIF) withdrawal on mESCs.
For this, he studied a total of 2717 mESCs with 24175 genes.
The data included 933 cells after LIF-withdrawal, 303 cells two days after, 683 cells 4 days after, and 798 cells 7 days after.
We normalize each cell by total counts over all genes, so that every cell has a total count equal to the median of total counts for observations (cells) before normalization, then we perform principal component analysis (PCA) and use the first three\todo{non-normalized?} principal components for clustering. 
To test the scalability of our new variants, we need larger data sets.
We use the well-known MNIST data set, with 784 features and 60000 samples
(PAMSIL will not be able to handle this size in reasonable time).
We implemented our algorithms in Rust, extending the \texttt{kmedoids} package~\cite{Schubert/22a},
wrapped with Python, and we make our source code available in this package.
We perform all computations in the same package,
to avoid side effects caused by comparing too different implementations~\cite{Kriegel/Schubert/17a}.
We run 10 restarts on an AMD EPYC 7302 processor using a single thread, and evaluate the average values.

\subsection{Clustering Quality}
We evaluated all methods with PAM BUILD initialization and a random initialization.
To evaluate the relevancy of the Average Silhouette Width and the Average Medoid Silhouette,
we compare true labels using the Adjusted Rand Index (ARI) and Normalized Mutual Information (NMI), two common measures in clustering.
\begin{table}[tb]\centering
\caption{Clustering results for the scRNA-seq data sets of Kolodziejczyk et al.~\cite{Kolodziejczyk/15a} for PAM, PAMSIL, and all variants of PAMMEDSIL. All methods are evaluated for BUILD and Random initialization, and true known $k$=3.}
\label{tab1}
\setlength{\tabcolsep}{6pt}
\begin{tabular}{ l|l|r|r|r|r|r } 
Algorithm & Initialization & AMS & ASW & ARI & NMI & run time (ms) \\
\hline
PAM & BUILD & 0.66 & 0.64 & 0.69 & 0.65 & 18.26 \\ 
PAM & Random & 0.66 & 0.64 & 0.69 & 0.65 & 22.67\\ 
PAMMEDSIL & BUILD & \bf0.67 & 0.65 & \bf0.72 & 0.70 & 62.63 \\ 
PAMMEDSIL & Random & \bf0.67 & 0.65 & \bf0.72 & 0.70 & 61.91 \\ 
FastMSC & BUILD & \bf0.67 & 0.65 & \bf0.72 & 0.70 & 25.09 \\
FastMSC & Random & \bf0.67 & 0.65 & \bf0.72 & 0.70 & 24.67 \\
FasterMSC & BUILD & \bf0.67 & 0.65 & \bf0.72 & 0.70 & \bf9.95 \\
FasterMSC & Random & \bf0.67 & 0.65 & \bf0.72 & 0.70 & 10.95 \\
PAMSIL & BUILD & 0.61 & \bf0.66 & \bf0.72 & \bf0.71 & 12493.86 \\
PAMSIL & Random & 0.61 & \bf0.66 & \bf0.72 & \bf0.71 & 16045.47 \\
\end{tabular}
\end{table}
On the data set from Kolodziejczyk shown in Table~\ref{tab1}, the highest ARI is achieved by the direct optimization methods for AMS and ASW. The different initialization provide the same results for all methods.
We get a much faster run time for the AMS variants compared to the ASW optimization.
For FasterMSC, we obtain the same ARI as for PAMSIL with 1255$\times$ faster run time and only a 0.01 lower NMI. As expected, AMS and ASW are optimal by those algorithms, that optimize for this measure,
but because the measures are correlated, those that optimize AMS only score 0.01 worse on the ASW.
Interestingly the total deviation used by PAM appears to be slightly more correlated to AMS than ASW
in this experiment.
Given the small difference, we argue that AMS is a suitable approximation for ASW,
at a much reduced run time.

Since there were no variations in the resulting medoids for the different restarts of the experiment, we can easily compare single results visually.
Figure~\ref{fig52} compares the results of PAMMEDSIL and PAMSIL, showing which
points are clustered differently than in the given labels.
Both clusters are similar, with class 1 captured better in one, classes 2 and 3 better in the other result.
\begin{figure}[tb]
\centering
\begin{subfigure}{0.45\textwidth}
\begin{tikzpicture}[font=\small]
\begin{axis}[unit vector ratio*=1 1 1, width=1.1\textwidth, xmin = -100, xmax = 150, ymin = -100, ymax = 150, xlabel={PC1}, ylabel={PC2}, xticklabels={,,}, yticklabels={,,}]
\addplot +[mark options={scale=0.6}, only marks, mark=text, text mark=1, fill=black, color = black] table [x=b, y=c, col sep=comma] {results/kolodziejczyk/zpammedsil_buildlabels_false.csv};
\addplot +[mark options={scale=0.6}, only marks, mark=text, text mark=1,Pa1] table [x=b, y=c, col sep=comma] {results/kolodziejczyk/zpammedsil_buildlabels_true.csv};
\addplot +[mark options={scale=0.6}, only marks, mark=text, text mark=2, fill=black, color = black] table [x=b, y=c, col sep=comma] {results/kolodziejczyk/zpammedsil_buildlabels1_false.csv};
\addplot +[mark options={scale=0.6}, only marks, mark=text, text mark=2, Pa2] table [x=b, y=c, col sep=comma] {results/kolodziejczyk/zpammedsil_buildlabels1_true.csv};
\addplot +[mark options={scale=0.6}, only marks, mark=text, text mark=3, fill=black, color = black] table [x=b, y=c, col sep=comma] {results/kolodziejczyk/zpammedsil_buildlabels2_false.csv};
\addplot +[mark options={scale=0.6}, only marks, mark=text, text mark=3, Pa3] table [x=b, y=c, col sep=comma] {results/kolodziejczyk/zpammedsil_buildlabels2_true.csv};
\end{axis}
\end{tikzpicture}
\caption{Results for PAMMEDSIL (BUILD)}
\label{plot2_1}
\end{subfigure}
\begin{subfigure}{0.45\textwidth}
\begin{tikzpicture}[font=\small]
\begin{axis}[unit vector ratio*=1 1 1, width=1.1\textwidth, xmin = -100, xmax = 150, ymin = -100, ymax = 150, xlabel={PC1}, ylabel={PC2}, xticklabels={,,}, yticklabels={,,}]
\addplot +[mark options={scale=0.6}, only marks, mark=text, text mark=1, fill=black, color = black] table [x=b, y=c, col sep=comma] {results/kolodziejczyk/zpamsil_buildlabels_false.csv};
\addplot +[mark options={scale=0.6}, only marks, mark=text, text mark=1,Pa1] table [x=b, y=c, col sep=comma] {results/kolodziejczyk/zpamsil_buildlabels_true.csv};
\addplot +[mark options={scale=0.6}, only marks, mark=text, text mark=2, fill=black, color = black] table [x=b, y=c, col sep=comma] {results/kolodziejczyk/zpamsil_buildlabels1_false.csv};
\addplot +[mark options={scale=0.6}, only marks, mark=text, text mark=2,Pa2] table [x=b, y=c, col sep=comma] {results/kolodziejczyk/zpamsil_buildlabels1_true.csv};
\addplot +[mark options={scale=0.6}, only marks, mark=text, text mark=3, fill=black, color = black] table [x=b, y=c, col sep=comma] {results/kolodziejczyk/zpamsil_buildlabels2_false.csv};
\addplot +[mark options={scale=0.6}, only marks, mark=text, text mark=3,Pa3] table [x=b, y=c, col sep=comma] {results/kolodziejczyk/zpamsil_buildlabels2_true.csv};
\end{axis}
\end{tikzpicture}
\caption{Results for PAMSIL (BUILD)}
\label{fig52}
\end{subfigure}
\caption{Clustering results for the scRNA-seq data sets of Kolodziejczyk et al.~\cite{Kolodziejczyk/15a} for PAMMEDSIL and PAMSIL. All correctly predicted labels are colored by the corresponding cluster and all errors are marked as black. }
\end{figure}
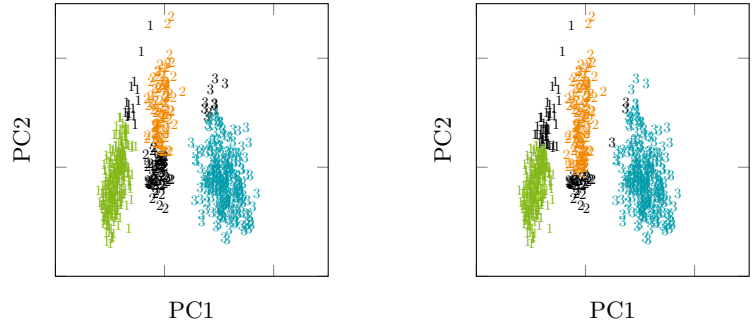
Table~\ref{tab2} shows the clustering results for the scRNA-seq data sets of Klein et al.~\cite{Klein/15a}.
In contrast to Kolodziejczyk's data set, we here obtain a higher ARI for PAMSIL than for the AMS optimization methods.\todo{was eigenartig ist, da die ASW fast gleich ist! Zufall?}
We get only the same high ARI and NMI for AMS optimization as for PAM, but a slightly higher ASW.
However, FasterMSC is 16521$\times$ faster than PAMSIL.

\begin{table}[tb]\centering
\caption{Clustering results for the scRNA-seq data sets of Klein et al.~\cite{Klein/15a} for PAM, PAMSIL and all variants of PAMMEDSIL. All methods are evaluated for BUILD and Random initialization and true known $k$=4.}
\label{tab2}
\setlength{\tabcolsep}{6pt}
\begin{tabular}{ l|l|r|r|r|r|r } 
Algorithm & Initialization & AMS & ASW & ARI & NMI & run time (ms) \\
\hline
PAM & BUILD & 0.75 & 0.82 & 0.84 & 0.87 & 355.55 \\ 
PAM & Random & 0.74 & 0.82 & 0.78 & 0.80 & 476.18\\ 
PAMMEDSIL & BUILD & \bf0.77 & 0.83 & 0.84 & 0.87 & 2076.15 \\ 
PAMMEDSIL & Random & \bf0.77 & 0.83 & 0.84 & 0.87 & 3088.77 \\ 
FastMSC & BUILD & \bf0.77 & 0.83 & 0.84 & 0.87 & 212.01 \\
FastMSC & Random & \bf0.77 & 0.83 & 0.84 & 0.87 & 305.00 \\
FasterMSC & BUILD & \bf0.77 & 0.83 & 0.84 & 0.87 & 163.74 \\
FasterMSC & Random & \bf0.77 & 0.83 & 0.84 & 0.87 & \bf122.63 \\
PAMSIL & BUILD & 0.67 & \bf0.84 & \bf0.95 & \bf0.92 & 2026025.10 \\
PAMSIL & Random & 0.67 & 0.84 & 0.93 & 0.91 & 1490354.10 \\
\end{tabular}
\end{table}

\subsection{Scalability}
To evaluate the scalability of our methods, we use the well-known MNIST data, which has 784 variables ($28{\times}28$ pixels) and 60000 samples. We use the first $n{=} 1000, \ldots, 30000$ samples and compare $k{=} 10$ and $k {=} 100$.
Due to its high run time, PAMSIL is not able to handle this size in a reasonable time.
In addition to the methods for direct AMS optimization, we evaluate the FastPAM1 and FasterPAM implementation.
For all methods we use random initialization.

\begin{figure}[tb!]\centering
    \begin{subfigure}{.49\textwidth}
	\begin{tikzpicture}[font=\tiny]
		\begin{axis}[
		    legend style={at={(.05,.95)},anchor=north west,fill=none,draw=none,inner sep=0,font=\tiny},
		    legend cell align={left},legend columns=2,
			height=23mm,
    	    width=\textwidth - 15mm,
    	    scale only axis,
		    every axis label/.style={inner sep=0, outer sep=0},
			xlabel = {number of samples},
			xmin = 1000, xmax = 30000,
			ylabel = {run time (s)},
			ymin = 0, ymax = 1000,
			yticklabel style={/pgf/number format/fixed},
			yticklabel style={/pgf/number format/1000 sep=},
	    	xtick={1000,5000,10000,15000,20000,25000,30000},
			xticklabel style={/pgf/number format/fixed},
			xticklabel style={/pgf/number format/1000 sep=},
			scaled x ticks=false
			]
			\addplot[Pa1, mark=triangle*]coordinates {
			(1000, 0.198)
			(5000, 7.075)
			(10000, 31.187)
			(15000, 62.150)
			(20000, 138.592)
			(25000, 172.462)
			(30000, 263.539)
			}; \label{plot_fpms}
			\addplot[Pa2, mark=triangle*]coordinates {
			(1000, 0.023)
			(5000, 0.531)
			(10000, 1.734)
			(15000, 5.238)
			(20000, 8.723)
			(25000, 10.502)
			(30000, 14.397)
			}; \label{plot_fepms}
			\addplot[Pa3, mark=diamond*]coordinates {
			(1000, 13.124)
			(5000, 212.825)
			(10000, 1657.221)
			(15000, 3298.933)
			(20000, 4847.659)
			(25000, 7343.586)
			(30000, 12568.364)
			}; \label{plot_pms}
			\addplot[Pa4, mark=oplus*]coordinates {
			(1000, 0.165)
			(5000, 1.056)
			(10000, 2.904)
			(15000, 3.473)
			(20000, 8.076)
			(25000, 16.969)
			(30000, 43.808)
			}; \label{plot_fp}
			\addplot[Pa5, mark=pentagon*]coordinates {
			(1000, 0.039)
			(5000, 2.213)
			(10000, 13.783)
			(15000, 30.483)
			(20000, 58.334)
			(25000, 137.264)
			(30000, 201.293)
			}; \label{plot_fp1}
			\addlegendimage{/pgfplots/refstyle=plot_fpms}\addlegendentry{FastMSC}
			\addlegendimage{/pgfplots/refstyle=plot_fepms}\addlegendentry{FasterMSC}
			\addlegendimage{/pgfplots/refstyle=plot_pms}\addlegendentry{PAMMEDSIL}
			\addlegendimage{/pgfplots/refstyle=plot_fp}\addlegendentry{FasterPAM}
			\addlegendimage{/pgfplots/refstyle=plot_fp1}\addlegendentry{FastPAM1}
		\end{axis}
	\end{tikzpicture}
	\caption{run time with $k{=}10$, linear scale}
	\end{subfigure}
	\hfill
    \begin{subfigure}{.49\textwidth}
	\begin{tikzpicture}[font=\tiny]
		\begin{axis}[
		    legend style={at={(0.70,0.35)},anchor=north,fill=none,draw=none,inner sep=0,font=\tiny},
		    legend cell align={left},legend columns=2,
    	    scale only axis,
		    every axis label/.style={inner sep=0, outer sep=0},
			height=23mm,
    	    width=\textwidth - 15mm,
			xlabel = {number of samples (log scale)},
			xmin = 1000, xmax = 30000,
			ylabel = {run time (s, log scale)},
			ymin = 0, ymax = 13000,
			ymode=log, xmode=log,
			yticklabel style={/pgf/number format/fixed},
			yticklabel style={/pgf/number format/1000 sep=},
			xticklabel style={/pgf/number format/fixed},
			xticklabel style={/pgf/number format/1000 sep=},
			scaled x ticks=false
			]
			\addplot[Pa1, mark=triangle*]coordinates {
			(1000, 0.198)
			(5000, 7.075)
			(10000, 31.187)
			(15000, 62.150)
			(20000, 138.592)
			(25000, 172.462)
			(30000, 263.539)
			}; \label{plot_fpms}
			\addplot[Pa2, mark=triangle*]coordinates {
			(1000, 0.023)
			(5000, 0.531)
			(10000, 1.734)
			(15000, 5.238)
			(20000, 8.723)
			(25000, 10.502)
			(30000, 14.397)
			}; \label{plot_fepms}
			\addplot[Pa3, mark=diamond*]coordinates {
			(1000, 13.124)
			(5000, 212.825)
			(10000, 1657.221)
			(15000, 3298.933)
			(20000, 4847.659)
			(25000, 7343.586)
			(30000, 12568.364)
			}; \label{plot_pms}
			\addplot[Pa4, mark=oplus*]coordinates {
			(1000, 0.165)
			(5000, 1.056)
			(10000, 2.904)
			(15000, 3.473)
			(20000, 8.076)
			(25000, 16.969)
			(30000, 43.808)
			}; \label{plot_fp}
			\addplot[Pa5, mark=pentagon*]coordinates {
			(1000, 0.039)
			(5000, 2.213)
			(10000, 13.783)
			(15000, 30.483)
			(20000, 58.334)
			(25000, 137.264)
			(30000, 201.293)
			}; \label{plot_fp1}
		\end{axis}
	\end{tikzpicture}
	\caption{run time with $k{=}10$, log-log plot}
	\end{subfigure}
	\\
	\begin{subfigure}{.49\textwidth}
	\begin{tikzpicture}[font=\tiny]
		\begin{axis}[
		    legend style={at={(0.70,0.35)},anchor=north,fill=none,draw=none,inner sep=0,font=\tiny},
		    legend cell align={left},legend columns=2,
    	    scale only axis,
		    every axis label/.style={inner sep=0, outer sep=0},
			height=23mm,
    	    width=\textwidth - 15mm,
			xlabel = {number of samples},
			xmin = 1000, xmax = 30000,
			ylabel = {run time (s)},
			ymin = 0, ymax = 2500,
			ytick={500,1000,1500,2000,2500},
			yticklabel style={/pgf/number format/fixed},
			yticklabel style={/pgf/number format/1000 sep=},
			xtick={1000,5000,10000,15000,20000,25000,30000},
			xticklabel style={/pgf/number format/fixed},
			xticklabel style={/pgf/number format/1000 sep=},
			scaled x ticks=false
			]
			\addplot[Pa1, mark=triangle*]coordinates {
			(1000, 0.21)
			(5000, 5.745)
			(10000, 22.975)
			(15000, 52.596)
			(20000, 93.356)
			(25000, 150.420)
			(30000, 210.180)
			}; \label{plot_fpms}
			\addplot[Pa2, mark=triangle*]coordinates {
			(1000, 0.01879)
			(5000, 0.661)
			(10000, 2.265)
			(15000, 4.025)
			(20000, 8.600)
			(25000, 11.023)
			(30000, 17.910)
			}; \label{plot_fepms}
			\addplot[Pa3, mark=diamond*]coordinates {
			(1000, 2014.386)
			(5000, 63854.893)
			(10000, 245499.702)
			}; \label{plot_pms}
			\addplot[Pa4, mark=oplus*]coordinates {
			(1000, 0.313)
			(5000, 1.171)
			(10000, 3.444)
			(15000, 5.547)
			(20000, 15.912)
			(25000, 32.377)
			(30000, 41.221)
			}; \label{plot_fp}
			\addplot[Pa5, mark=pentagon*]coordinates {
			(1000, 0.149)
			(5000, 5.254)
			(10000, 13.345)
			(15000, 30.426)
			(20000, 53.143)
			(25000, 85.543)
			(30000, 135.455)
			}; \label{plot_fp1}
		\end{axis}
	\end{tikzpicture}
	\caption{run time with $k{=}100$, linear scale}
	\end{subfigure}
	\hfill
	\begin{subfigure}{.49\textwidth}
	\begin{tikzpicture}[font=\tiny]
		\begin{axis}[
		    legend style={at={(0.70,0.35)},anchor=north,fill=none,draw=none,inner sep=0,font=\tiny},
		    legend cell align={left},legend columns=2,
    	    scale only axis,
		    every axis label/.style={inner sep=0, outer sep=0},
			height=23mm,
    	    width=\textwidth - 15mm,
			xlabel = {number of samples (log scale)},
			xmin = 1000, xmax = 30000,
			ylabel = {run time (s, log scale)},
			ymin = 0, ymax = 300000,
			ymode=log, xmode=log,
			yticklabel style={/pgf/number format/fixed},
			yticklabel style={/pgf/number format/1000 sep=},
			xticklabel style={/pgf/number format/fixed},
			xticklabel style={/pgf/number format/1000 sep=},
			scaled x ticks=false
			]
			\addplot[Pa1, mark=triangle*]coordinates {
			(1000, 0.21)
			(5000, 5.745)
			(10000, 22.975)
			(15000, 52.596)
			(20000, 93.356)
			(25000, 150.420)
			(30000, 210.180)
			}; \label{plot_fpms}
			\addplot[Pa2, mark=triangle*]coordinates {
			(1000, 0.01879)
			(5000, 0.661)
			(10000, 2.265)
			(15000, 4.025)
			(20000, 8.600)
			(25000, 11.023)
			(30000, 17.910)
			}; \label{plot_fepms}
			\addplot[Pa3, mark=diamond*]coordinates {
			(1000, 2014.386)
			(5000, 63854.893)
			(10000, 245499.702)
			}; \label{plot_pms}
			\addplot[Pa4, mark=oplus*]coordinates {
			(1000, 0.190)
			(5000, 1.776)
			(10000, 3.976)
			(15000, 5.571)
			(20000, 14.372)
			(25000, 36.438)
			(30000, 26.592)
			}; \label{plot_fp}
			\addplot[Pa5, mark=pentagon*]coordinates {
			(1000, 0.149)
			(5000, 5.254)
			(10000, 13.345)
			(15000, 30.426)
			(20000, 53.143)
			(25000, 85.543)
			(30000, 135.455)
			}; \label{plot_fp1}
		\end{axis}
	\end{tikzpicture}
	\caption{run time with $k{=}100$, log-log plot}
	\end{subfigure}
	\caption{Run time on MNIST data (time out 24 hours)}
	\label{fig5}
\end{figure}
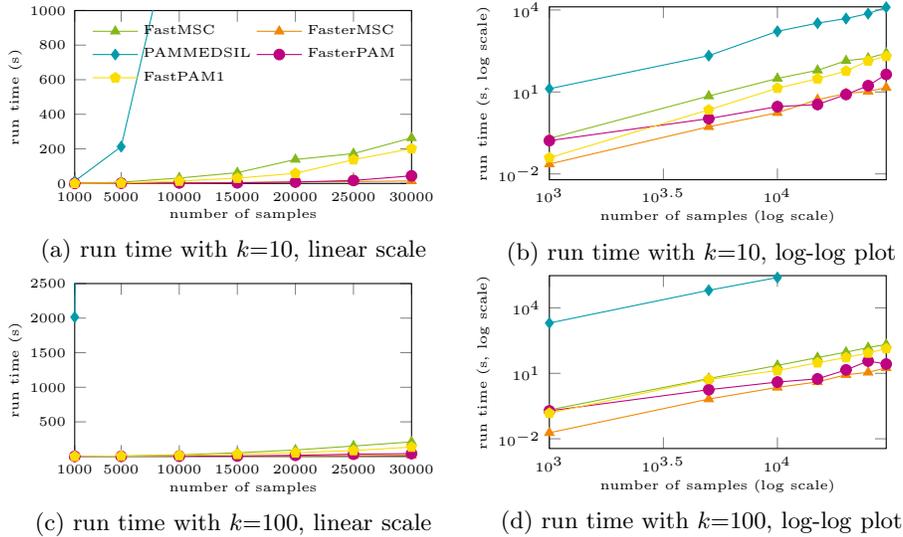
As expected, all methods scale approximately quadratic in the sample size $n$.
FastMSC is on average 50.66x faster than PAMMEDSIL for $k{=}10$ and 10464.23$\times$ faster for $k{=}100$,
supporting the expected $O(k^2)$ improvement by removing the nested loop and caching the distances to the nearest centers.
For FasterMSC we achieve even 639.34$\times$ faster run time than for PAMMEDSIL for $k$=10 and 78035.01$\times$ faster run time for $k$=100.
We expect FastPAM1 and FastMSC and also FasterPAM and FasterMSC to have similar scalability;
but since MSC needs additional bounds it needs to maintain more data and access more memory.
We observe that FastPAM1 is 2.50$\times$ faster than FastMSC for $k{=}10$ and 1.57$\times$ faster for $k{=}100$,
which is larger than expected and due to more iterations necessary for convergence in the MSC methods: FastPAM1 needs on average 14.86 iterations while FastMSC needs 33.48.
In contrast, FasterMSC is even 1.65$\times$ faster than FasterPAM for $k{=}10$ and 1.96$\times$ faster for $k{=}100$.

\section{Conclusions}
We showed that the Average Medoid Silhouette satisfies desirable
theoretical properties for clustering quality measures,
and as an approximation of the Average Silhouette Width
yields desirable results on real problems from gene expression analysis.
We propose a new algorithm for optimizing the Average Medoid Silhouette,
which provides a run time speedup of $O(k^2)$
compared to the earlier PAMMEDSIL algorithm
by avoiding unnecessary distance computations via caching of the distances to the nearest centers
and of partial results based on FasterPAM.
This makes clustering by optimizing the Medoid Silhouette possible on much larger data sets than before.
The ability to optimize a variant of the popular Silhouette measure directly
demonstrates the underlying problem that any internal cluster evaluation measure specifies a clustering itself.

\vfill\pagebreak
\bibliographystyle{splncs04}
\bibliography{references.bib}

\begin{thebibliography}{10}
\providecommand{\url}[1]{\texttt{#1}}
\providecommand{\urlprefix}{URL }
\providecommand{\doi}[1]{https://doi.org/#1}

\bibitem{DBLP:conf/nips/Ben-DavidA08}
Ackerman, M., Ben{-}David, S.: Measures of clustering quality: {A} working set
  of axioms for clustering. In: NIPS (2008)

\bibitem{DBLP:journals/pr/ArbelaitzGMPP13}
Arbelaitz, O., Gurrutxaga, I., Muguerza, J., P{\'{e}}rez, J.M., Perona, I.: An
  extensive comparative study of cluster validity indices. Pattern Recognit.
  \textbf{46}(1) (2013)

\bibitem{Batool/Hennig/21a}
Batool, F., Hennig, C.: Clustering with the average silhouette width.
  Computational Statistics and Data Analysis  \textbf{158} (2021)

\bibitem{DBLP:journals/ibmrd/Bonner64}
Bonner, R.E.: On some clustering techniques. {IBM} Journal of Research and
  Development  \textbf{8}(1) (1964). \doi{10.1147/rd.81.0022}

\bibitem{DBLP:journals/pr/BrunSHLCSD07}
Brun, M., Sima, C., Hua, J., Lowey, J., Carroll, B., Suh, E., Dougherty, E.R.:
  Model-based evaluation of clustering validation measures. Pattern Recognit.
  \textbf{40}(3) (2007). \doi{10.1016/j.patcog.2006.06.026}

\bibitem{DBLP:conf/kdd/EsterKSX96}
Ester, M., Kriegel, H.P., Sander, J., Xu, X.: A density-based algorithm for
  discovering clusters in large spatial databases with noise. In: KDD'96 (1996)

\bibitem{DBLP:journals/sigkdd/Estivill-Castro02}
Estivill-Castro, V.: Why so many clustering algorithms -- a position paper.
  SIGKDD Explorations  \textbf{4}(1) (2002). \doi{10.1145/568574.568575}

\bibitem{Kaufman/Rousseeuw/87a}
Kaufman, L., Rousseeuw, P.J.: Clustering by means of medoids. In: Statistical
  Data Analysis Based on the $L_1$ Norm and Related Methods. North-Holland
  (1987)

\bibitem{Kaufman/Rousseeuw/90c}
Kaufman, L., Rousseeuw, P.J.: Finding Groups in Data, chap. Clustering Large
  Applications (Program {CLARA}). Wiley (1990)

\bibitem{Klein/15a}
Klein, A., Mazutis, L., Akartuna, I., Tallapragada, N., Veres, A., Li, V.,
  Peshkin, L., Weitz, D., Kirschner, M.: Droplet barcoding for single-cell
  transcriptomics applied to embryonic stem cells. Cell  \textbf{161} (2015)

\bibitem{Kleinberg/Jon/02a}
Kleinberg, J.: An impossibility theorem for clustering. In: NIPS. vol.~15
  (2002)

\bibitem{Kolodziejczyk/15a}
Kolodziejczyk, A., Kim, J., Tsang, J., Ilicic, T., Henriksson, J., Natarajan,
  K., Tuck, A., Gao, X., Bühler, M., Liu, P., Marioni, J., Teichmann, S.:
  Single cell {RNA}-sequencing of pluripotent states unlocks modular
  transcriptional variation. Cell Stem Cell  \textbf{17}(4) (2015).
  \doi{10.1016/j.stem.2015.09.011}

\bibitem{Kriegel/Schubert/17a}
Kriegel, H.P., Schubert, E., Zimek, A.: The (black) art of runtime evaluation:
  Are we comparing algorithms or implementations? Knowledge and Information
  Systems  \textbf{52} (2017). \doi{10.1007/s10115-016-1004-2}

\bibitem{Rousseeuw/87a}
Rousseeuw, P.J.: Silhouettes: A graphical aid to the interpretation and
  validation of cluster analysis. J. Comput. Appl. Math.  \textbf{20} (1987)

\bibitem{DBLP:journals/tods/SchubertSEKX17}
Schubert, E., Sander, J., Ester, M., Kriegel, H.P., Xu, X.: {DBSCAN} revisited,
  revisited: Why and how you should (still) use {DBSCAN}. {ACM} Trans. Database
  Syst.  \textbf{42}(3) (2017). \doi{10.1145/3068335}

\bibitem{DBLP:conf/lwa/SchubertHM18}
Schubert, E., Hess, S., Morik, K.: The relationship of {DBSCAN} to matrix
  factorization and spectral clustering. In: Lernen, Wissen, Daten, Analysen
  (2018)

\bibitem{Schubert/22a}
Schubert, E., Lenssen, L.: Fast k-medoids clustering in {R}ust and {P}ython.
  Journal of Open Source Software  \textbf{7}(75) (2022).
  \doi{10.21105/joss.04183}

\bibitem{DBLP:conf/sisap/SchubertR19}
Schubert, E., Rousseeuw, P.J.: Faster k-medoids clustering: Improving the
  {PAM}, {CLARA}, and {CLARANS} algorithms. In: Int. Conf. Similarity Search
  and Applications, {SISAP} (2019). \doi{10.1007/978-3-030-32047-8\_16}

\bibitem{DBLP:journals/is/SchubertR21}
Schubert, E., Rousseeuw, P.J.: Fast and eager k-medoids clustering: {O(k)}
  runtime improvement of the {PAM}, {CLARA}, and {CLARANS} algorithms. Inf.
  Syst.  \textbf{101} (2021)

\bibitem{VanderLaan/03a}
{Van der Laan}, M., Pollard, K., Bryan, J.: A new partitioning around medoids
  algorithm. Journal of Statistical Computation and Simulation  \textbf{73}(8)
  (2003). \doi{10.1080/0094965031000136012}

\end{thebibliography}

\end{document}